\documentclass[letterpaper]{article} % DO NOT CHANGE THIS
\usepackage{aaai2027}  % DO NOT CHANGE THIS
\usepackage[hyphens]{url}  % DO NOT CHANGE THIS
\usepackage{graphicx} % DO NOT CHANGE THIS
\usepackage{natbib}  % DO NOT CHANGE THIS AND DO NOT ADD ANY OPTIONS TO IT
\usepackage{caption} % DO NOT CHANGE THIS AND DO NOT ADD ANY OPTIONS TO IT
\usepackage{booktabs}
\usepackage{multirow}

\usepackage{amsmath}
\usepackage{amssymb}
\usepackage{amsthm}

\newcommand{\method}{AuroOFT}
\newcommand{\qoft}{QOFT}
\newcommand{\R}{\mathbb{R}}

\newcommand{\key}[1]{\textbf{\emph{#1}}}

\newtheorem{proposition}{Proposition}
\newtheorem{remark}{Remark}

\title{Beyond Rotations: \method\ for Expressive Quantized Orthogonal Fine-Tuning}

\author{
    Yue Han\textsuperscript{\rm 1},
    Dianlin Wang\textsuperscript{\rm 2},
    Xinkang Li\textsuperscript{\rm 2},
    Jie Zhang\textsuperscript{\rm 1},\\
    Ziyi Chen\textsuperscript{\rm 1},
    Tao Wang\textsuperscript{\rm 1},
    Yexin Cui\textsuperscript{\rm 2},
    Weihong Han\textsuperscript{\rm 3}
}
\affiliations{
    \textsuperscript{\rm 1}College of Systems Engineering, National University of Defense Technology\\
    \textsuperscript{\rm 2}College of Computer Science and Technology, National University of Defense Technology\\
    \textsuperscript{\rm 3}Vernal Institute, China Electronics Corporation\\
    \{hanyue, wangdianlin, xinkangli, zhangjie, wangtao1976, yexincuicyx\}@nudt.edu.cn\\
    chenziyi\_nudt@outlook.com,
    hanwh@pcl.ac.cn
}

\begin{document}

\maketitle
\raggedbottom
\setlength{\textfloatsep}{6pt plus 2pt minus 2pt}
\setlength{\dbltextfloatsep}{6pt plus 2pt minus 2pt}
\setlength{\intextsep}{6pt plus 2pt minus 2pt}

\begin{abstract}
Quantized orthogonal fine-tuning (\qoft) enables parameter-efficient adaptation of low-bit language models by learning structured activation rotations before frozen quantized weights. However, its task-specific updates remain constrained to linear orthogonal transformations, limiting input-dependent nonlinear corrections. We introduce \method, which keeps \qoft\ as a stable quantization-compatible branch while attaching a zero-start gated low-rank nonlinear residual to each adapted linear layer. \method\ maps activations into an RMS-normalized compact latent space and uses adaptive nonlinear bases with bounded or token-dependent gating. The zero-initialized up projection makes \method\ functionally identical to \qoft\ at initialization, while orthogonality remains a branch-level stability property rather than a property of the combined nonlinear layer. Under matched data, optimization, decoding, and parser protocols, \method\ improves Macro-6 over matched \qoft\ by 1.30--2.70 points on the 1.5B/3B Qwen2.5 settings, exceeds QLoRA by 6.52--10.62 points, and saves 32.3--44.7\% trainable parameters relative to QLoRA in representative scales. The small exam-style multiple-choice math set is treated only as a protocol-sensitivity diagnostic. Our code is available at the anonymous repository \url{https://anonymous.4open.science/r/AuroOFT-F3FD}.
\end{abstract}

\section{Introduction}

Parameter-efficient fine-tuning has become the standard route for adapting large models under realistic compute and storage budgets~\citep{lialin2023scaling}. For low-bit reasoning models, however, an effective adapter should behave like a \key{small but faithful correction layer}. It must fit beside a frozen quantized foundation model, preserve useful pretrained geometry, and still express the input-dependent corrections required by mathematical reasoning. Figure~\ref{fig:research_problem} summarizes this tension. Additive adapters can be flexible but need not preserve geometry. Orthogonal adapters can be stable but may restrict the local functions available for adaptation. The central question is therefore not only whether an adapter can be made small, but what kind of correction a small adapter is allowed to express.

\begin{figure}[t]
    \centering
    \includegraphics[width=\columnwidth]{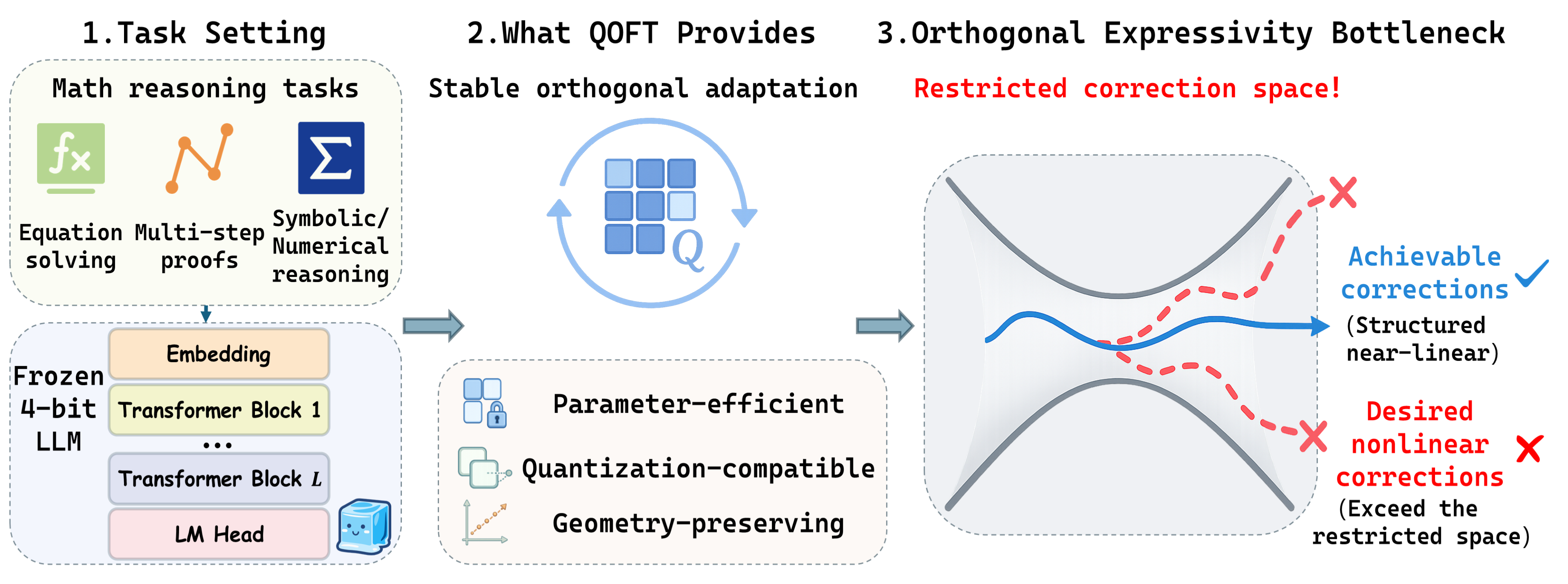}
    \caption{Research problem and challenge. Low-bit reasoning adapters must balance quantized stability with expressive correction capacity.}
    \label{fig:research_problem}
\end{figure}

This issue is especially visible in mathematical reasoning. A model may preserve broad pretrained knowledge and still fail because a local symbolic step, choice elimination, or numerical comparison requires an input-dependent correction. LoRA and QLoRA obtain efficiency by constraining updates to compact low-rank linear subspaces~\citep{hu2022lora,dettmers2023qlora}. OFT, OFTv2, and \qoft\ instead use structure-preserving orthogonal transformations, with \qoft\ applying an input-side Cayley-Neumann rotation before frozen NF4 weights~\citep{qiu2023oft,qiu2025oftv2}. This gives a stable quantized path, but the task-specific change remains an orthogonal linear transformation. Nonlinear low-rank methods such as LoRAN and AuroRA show that compact adapters need not remain linear~\citep{li2024loran,dong2025aurora}, yet they are designed around LoRA-style additive updates rather than the geometric backbone that motivates \qoft.

We call the resulting gap the \key{orthogonal expressivity bottleneck}. QOFT is attractive because its rotations are structured, quantization-compatible, and conservative around frozen low-bit weights. The same design, however, asks each adapted layer to express task change primarily through an input-side orthogonal linear family. For reasoning, this scope can be too narrow. The desired correction may depend on the current token state, the latent arithmetic context, or a nonlinear interaction that is hard to represent by rotation alone. Simply inserting a nonlinear module into the rotation path would blur the geometric meaning of QOFT, while replacing QOFT would discard the stability that made it useful.

\emph{Can we keep \qoft's quantized orthogonal path as the stable backbone while giving each layer expressive nonlinear corrections?} Our answer is \method, a zero-start nonlinear residual augmentation for quantized orthogonal fine-tuning. Each adapted layer keeps the \qoft\ branch unchanged. Activations are rotated by the learned Cayley-Neumann adapter and then passed through the frozen NF4 weight. In parallel, a lower branch applies a down projection, an ANL, a zero-initialized up projection, and a gate. The branches meet by residual addition (Figure~\ref{fig:methodology}).

This decomposition gives \method\ its central principle. It treats \key{orthogonality as a stable quantized backbone, nonlinearity as a zero-start residual correction}. The nonlinear branch expands the layer-wise function family without modifying the quantized base weight or inserting nonlinear operations inside the Cayley-Neumann rotation. Zero initialization of the up projection makes the initial model exactly function-equivalent to \qoft, while the input-dependent residual deliberately turns the combined layer into a nonlinear mapping. Thus, only the \qoft\ branch preserves orthogonal structure. \method\ therefore borrows the nonlinear-low-rank principle behind AuroRA, but changes the augmented object. The nonlinear map is placed beside a quantized orthogonal \qoft\ layer, not in place of it.

This distinction also shapes how the method should be evaluated. Improvements over QLoRA test whether the parameter budget is used more effectively than a standard low-rank additive adapter. Improvements over matched QOFT test whether the nonlinear residual supplies useful expressivity beyond the stable orthogonal carrier. We therefore emphasize comparisons under fixed data, optimizer, decoding, and parser conditions, and treat small or protocol-sensitive benchmarks as diagnostic rather than as standalone evidence for a structural claim.

The broader point is that \method\ is not a request to abandon orthogonality, nor a claim that nonlinear adapters are universally better. It is a more specific proposal about where expressivity should enter a quantized orthogonal adapter. The orthogonal branch remains responsible for stable low-bit adaptation. The residual branch handles local, input-dependent deviations that the rotation family may not capture. This separation lets the method preserve a clean baseline computation at initialization, expose a controlled set of AuroOFT-specific variables for ablation, and keep the scientific comparison focused on the added correction family. In this sense, the method is deliberately conservative in its starting point, but less restrictive in the corrections it can learn during optimization.

This paper contributes a compact framework for turning the above diagnosis into a testable adaptation object.
\begin{itemize}
    \item \textbf{A diagnosis.} The \key{orthogonal expressivity bottleneck} isolates the central limitation. Input-side rotations give \qoft\ a stable quantized path, yet restrict task updates to a structured linear family.
    \item \textbf{An adaptation object.} \method\ instantiates this diagnosis as a zero-start gated nonlinear low-rank residual branch parallel to \qoft, with tanh, normalized spline, and enhanced dual-ANL variants.
    \item \textbf{A boundary analysis.} The formulation makes explicit what is preserved and what is not. It covers exact initialization equivalence to \qoft, containment of the \qoft\ family, branch-limited orthogonality, and non-mergeable residual cost during inference.
    \item \textbf{A matched evaluation.} Mathematical reasoning experiments align data, optimization, decoding, and parsers, while treating the small exam-style multiple-choice math set only as a protocol diagnostic. The protocol uses Macro-6 to separate structural gains over \qoft\ from parameter-efficiency comparisons with QLoRA.
\end{itemize}

\section{Related Work}
\label{sec:related}

\paragraph{Stable and quantized adaptation.}
PEFT methods differ in the adaptation object they choose. These include bottleneck modules~\citep{houlsby2019adapters}, prompts~\citep{li2021prefix,liu2022ptuningv2}, activation scalings~\citep{liu2022ia3}, and low-rank additive matrices~\citep{hu2022lora}. Quantized adaptation makes this choice sharper. GPTQ, SmoothQuant, and AWQ reduce the frozen model footprint~\citep{frantar2023gptq,xiao2023smoothquant,lin2024awq}, while QLoRA, LoftQ, and QA-LoRA attach LoRA-style updates to low-bit weights~\citep{dettmers2023qlora,li2024loftq,xu2024qalora}.

Orthogonal PEFT changes the prior by learning geometry-preserving rotations. OFT introduces orthogonal fine-tuning~\citep{qiu2023oft}, BOFT compresses rotations with butterfly factors~\citep{liu2024boft}, and OFTv2/\qoft\ scales the idea through input-centric Cayley-Neumann rotations before frozen quantized weights~\citep{qiu2025oftv2}. These lines emphasize efficiency, low-bit compatibility, and stability, but still leave a common gap for reasoning. The task correction is usually a linear additive map or an input-side orthogonal rotation. AuroOFT keeps the quantized orthogonal carrier while adding a separate nonlinear correction object.

\paragraph{Expressive residual adaptation.}
Expressive PEFT asks what a compact adapter can represent, not only how many parameters it adds. LoRA analyses and variants such as AdaLoRA, DyLoRA, DoRA, VeRA, FourierFT, MiLoRA, MoSLoRA, and PeriodicLoRA enrich low-rank updates through rank allocation, shared bases, magnitude-direction decomposition, spectral coefficients, mixed subspaces, or staged rank growth~\citep{zeng2023expressive,zhang2023adalora,valipour2023dylora,liu2024dora,kopiczko2024vera,gao2024fourierft,wang2025milora,wu2024moslora,meng2024periodiclora}. LoRAN and AuroRA further insert nonlinear hidden mappings, including tanh and B-spline components, into the low-rank path~\citep{li2024loran,dong2025aurora}. These works show that compact adaptation need not be purely linear, but they are designed around LoRA-style additive updates rather than a quantized orthogonal main branch.

AuroOFT uses their expressivity lesson in a different role. The nonlinear map is zero-started, gated, and residual beside QOFT, so ablations can vary only the correction family while the orthogonal carrier remains fixed. This distinction is important for attribution. The proposed residual is not an alternative backbone, but a controlled way to test whether nonlinear corrections can expand QOFT without changing its quantized orthogonal carrier. The comparison therefore follows a narrow principle. Preserve the stable branch, vary the residual branch, and measure whether the extra expressivity improves adaptation under matched conditions. This principle bridges prior adapter design to the AuroOFT layer formalized next.

\begin{figure*}[t]
    \centering
    \includegraphics[width=.96\textwidth,height=.31\textheight]{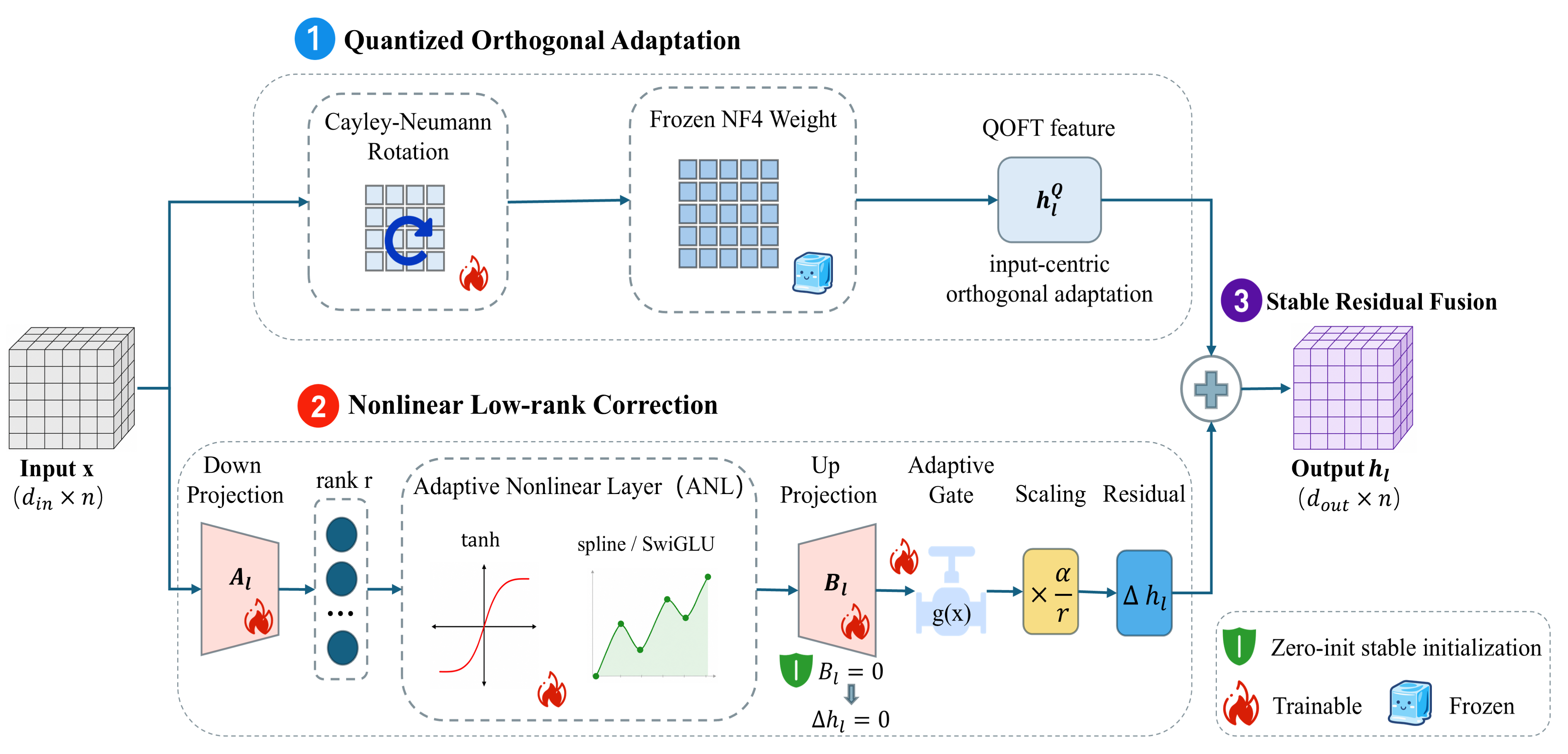}
    \caption{\method\ layer. A stable \qoft\ branch is augmented by a zero-start gated nonlinear low-rank residual.}
    \label{fig:methodology}
\end{figure*}

\section{Preliminaries}
\label{sec:prelim}

\paragraph{QOFT.}
Let $x\in\R^{d_{\rm in}}$ be the input activation of a linear layer with frozen quantized weight $W_{q,l}$. \qoft\ applies an orthogonal transformation to the activation side before the quantized linear map.
\begin{equation}
    h_l^{Q}
    =
    \operatorname{Dequant}(W_{q,l})^\top R_l^\top x .
    \label{eq:qoft}
\end{equation}
The rotation $R_l$ is parameterized by a skew-symmetric matrix through a Cayley transform. OFTv2 replaces the inverse in the Cayley transform with a truncated Neumann series, and the implementation used in our experiments follows the common setting of block size $32$ and five Neumann terms. The important point for \method\ is that the orthogonal adapter acts before the frozen NF4 weight. The quantized base weight is not rotated or modified.

\paragraph{AuroRA-style nonlinear low-rank mapping.}
LoRA uses a two-projector linear update.
\begin{equation}
    h = W_0 x + BAx ,
\end{equation}
which is still linear in $x$ for fixed $A$ and $B$. AuroRA inserts an adaptive nonlinear layer $\sigma(\cdot)$ between the two projectors.
\begin{equation}
    h = W_0 x + B\sigma(Ax).
\end{equation}
Its ANL combines fixed nonlinear activations and learnable B-spline components. \method\ adopts this nonlinear-low-rank principle, but it does not replace a LoRA branch. It uses the nonlinear mapping as a residual beside \qoft.

\section{Method}
\label{sec:method}

\subsection{AuroOFT Layer}

Figure~\ref{fig:methodology} shows the layer-level construction, and Algorithm~1 summarizes the corresponding forward computation. \method\ wraps an already adapted OFT/\qoft\ linear module rather than replacing it. The upper branch is the stable quantized orthogonal path inherited from \qoft. The lower branch is a nonlinear low-rank residual evaluated on the same activation and added only at the output.

For an adapted layer $l$, \method\ computes the following branch input and residual output.
\begin{equation}
    z_l = A_lx,\qquad
    \tilde z_l=\mathcal{N}_l(z_l),
    \label{eq:branch_input}
\end{equation}
\begin{equation}
    \begin{aligned}
    h_l &= h_l^Q + \Delta h_l,\\
    \Delta h_l
    &= G_l(\tilde z_l)\frac{\alpha}{r}B_l
    \operatorname{Drop}\!\left(\phi_l(\tilde z_l)\right),
    \end{aligned}
    \label{eq:aurooft}
\end{equation}
where $h_l^Q$ is the \qoft\ output in Eq.~\ref{eq:qoft}, $A_l\in\R^{r\times d_{\rm in}}$ is the down projection, $B_l\in\R^{d_{\rm out}\times r}$ is the up projection, $\mathcal{N}_l$ is either the identity map or FP32 RMS normalization, $\phi_l$ is an adaptive nonlinear layer in the compact rank-$r$ space, $\alpha/r$ is the residual scaling factor, and $G_l$ is a scalar, bounded scalar, or token-dependent gate. This matches the implementation. An already adapted OFT/\qoft\ module is wrapped, the nonlinear branch is evaluated on the same input activation, and its gated residual is added to the wrapped module output.

This parallel placement is the defining design choice. The nonlinear branch is not inside the Cayley-Neumann rotation and does not change the orthogonal parameterization. It provides an additional input-dependent correction after the \qoft\ feature has been computed. The zero-initialized up projection ensures that this additional path starts inactive, while the gate and $\alpha/r$ scaling regulate how the residual enters the layer during training.

\begin{figure}[t]
\refstepcounter{figure}\label{alg:aurooft_forward}
\centering
\begin{minipage}{0.96\columnwidth}
\small
\hrule
\vspace{2pt}
\noindent\textbf{Algorithm 1} \method\ Layer Forward Pass
\vspace{2pt}
\hrule
\vspace{3pt}
\begin{tabular}{@{}r@{\hspace{0.55em}}p{0.86\columnwidth}@{}}
1: & \textbf{Input.} activation $x$, wrapped \qoft\ layer $f_l^Q$, branch modules $\{A_l,\mathcal{N}_l,\phi_l,B_l,G_l\}$ \\
2: & \textbf{Output.} adapted feature $h_l$ \\
3: & $h_l^Q \leftarrow f_l^Q(x)$ \hfill \textit{stable QOFT branch} \\
4: & $z_l \leftarrow A_lx$ \hfill \textit{down projection} \\
5: & $\tilde z_l \leftarrow \mathcal{N}_l(z_l)$ \hfill \textit{identity or FP32 RMSNorm} \\
6: & $u_l \leftarrow \operatorname{Drop}(\phi_l(\tilde z_l))$ \hfill \textit{ANL in rank-$r$ space} \\
7: & $g_l \leftarrow G_l(\tilde z_l)$ \hfill \textit{scalar, bounded, or token gate} \\
8: & $\Delta h_l \leftarrow g_l \frac{\alpha}{r} B_lu_l$ \\
9: & $h_l \leftarrow h_l^Q+\Delta h_l$ \\
10: & \textbf{return} $h_l$ \\
11: & \textbf{Zero-start.} Initialize $B_l=0\Rightarrow\Delta h_l=0$, so the layer starts exactly as \qoft. \\
\end{tabular}
\vspace{2pt}
\hrule
\end{minipage}
\end{figure}

\subsection{Lite ANL}

The lightweight variant uses a two-tanh hidden transformation plus a learnable per-dimension residual scale.
\begin{equation}
    \phi_l^{\rm Lite}(z)
    =
    \tanh(H_l\tanh(z)) + s_l\odot\tanh(z),
    \label{eq:lite}
\end{equation}
where $z=\tilde z_l$, $H_l\in\R^{r\times r}$, and $s_l\in\R^r$. The code initializes $H_l$ as the identity and $s_l$ to a small positive value. This variant is computationally simple and isolates whether nonlinear hidden mixing alone is useful beyond \qoft.

\subsection{SplineNorm ANL}

The normalized spline variant follows the AuroRA motivation more closely while adapting it to the \method\ branch.
\begin{equation}
    \bar z = \tanh(\operatorname{RMSNorm}_{\rm FP32}(z)),
\end{equation}
\begin{equation}
    \phi_l^{\rm SplineNorm}(z)
    =
    \tanh(H_l\tanh(\bar z)) +
    C_l\Psi(\bar z),
    \label{eq:splinenorm}
\end{equation}
where $\mathcal{B}_j(\cdot)$ denotes the $j$-th B-spline basis response on the bounded interval $[-1,1]$, $\Psi_j(\bar z)=\sum_{i=1}^{r}\mathcal{B}_j(\bar z_i)$ is the implementation's rank-dimension aggregated spline feature, and $C_l\in\R^{r\times (M+p)}$ is the learnable spline coefficient matrix for grid size $M$ and order $p$. The FP32 RMS normalization and tanh bounding are used to keep low-dimensional activations within the spline grid. This is an adaptation of the AuroRA ANL idea, not a direct transplantation of its LoRA theory.

\subsection{Enhanced ANL and Gating}

The enhanced branch uses an FP32-normalized compact representation
\begin{equation}
    z = \operatorname{RMSNorm}_{\rm FP32}(A_lx),
\end{equation}
and combines two nonlinear bases.
\begin{align}
    u_t &= \tanh(P_{t,l}\tanh z), \\
    u_g &= \operatorname{SiLU}(P_{g,l}z)\odot\tanh(P_{v,l}z).
\end{align}
For the dual ANL, the final hidden output is
\begin{align}
    \phi_l^{\rm Dual}(z)
    &=
    \pi_{t,l}u_t+\pi_{g,l}u_g, \\
    [\pi_{t,l},\pi_{g,l}]
    &=
    \operatorname{softmax}(m_l).
    \label{eq:dual}
\end{align}
The code supports scalar, bounded, and token-dependent gates.
\begin{align}
    G_l(z)&=\eta_l,\\
    G_l(z)&=2\sigma(\gamma_l),\\
    G_l(z)&=2\sigma(W_{g,l}z),
    \label{eq:gates}
\end{align}
corresponding respectively to the unconstrained scalar, bounded scalar, and token-dependent gate. In the enhanced implementation, $\eta_l$ is initialized to one, while $\gamma_l=0$ and $W_{g,l}=0$ make the bounded and token gates initialize to one. This avoids the small-gate start that can suppress early gradients through the zero-initialized up projection.

\subsection{Layer Injection and Optimization}

\method\ supports all-layer injection or selective targets such as query/value and query/key/value plus MLP up/down projections. The \qoft\ branch and nonlinear residual use separate parameter groups, with branch-specific learning rate and gradient clipping. Parameters are identified by module identity rather than name, avoiding accidental capture of unrelated \qoft\ modules. Disabling \method\ recovers the same \qoft\ training interface.

\subsection{Structural Properties}
\label{sec:properties}

\begin{proposition}[Zero-start equivalence]
If $B_l=0$ for every nonlinear branch, then the initialized \method\ model computes exactly the same function as the corresponding \qoft\ model on every input sequence.
\end{proposition}

\begin{proof}
For each adapted layer, $B_l=0$ implies $\Delta h_l=0$ in Eq.~\ref{eq:aurooft}, regardless of $A_l$, $\phi_l$, dropout, or gate values. Hence $h_l=h_l^Q$ at every adapted layer, and the layer-wise equality composes through the network.
\end{proof}

\begin{proposition}[QOFT containment]
The function family induced by \method\ contains the corresponding \qoft\ family when the nonlinear residual is inactive.
\end{proposition}

\begin{proof}
Any \qoft\ model can be represented by setting the nonlinear residual to zero, for example through $B_l=0$ at all adapted layers. Thus \qoft\ is a subfamily of \method. We use the up-projection condition because it holds for all implemented gate types, including bounded and token-dependent gates whose finite-parameter values are positive.
\end{proof}

\begin{remark}[What the containment does not prove]
Containment is only an expressivity statement. Under ideal global optimization, the best empirical objective is no worse than \qoft. It does not guarantee finite-step accuracy gains or justify test-set-based selection.
\end{remark}

\paragraph{Orthogonality boundary.}
The main branch retains \qoft's approximate orthogonal structure, but the full \method\ layer adds an input-dependent nonlinear residual and is not an orthogonal transformation. We therefore do not claim global angle, norm, or forgetting guarantees for the combined mapping.

\paragraph{Non-mergeability and cost.}
Because $\Delta h_l$ is nonlinear in the current activation, the residual cannot generally be merged into a static post-training weight matrix. Unlike linear LoRA-style adapters, \method\ therefore adds inference cost from projection, nonlinear computation, gating, and up projection.

\begin{table*}[t]
\centering
\small
\setlength{\tabcolsep}{2.2pt}
\begin{tabular}{@{}llcccccccc@{}}
\toprule
Model & Type & Params & AMC23 & AQUA & CMATH & GaoKao & Minerva & Olymp. & SAT \\
\midrule
\multirow{4}{*}{Qwen2.5-1.5B-it} & Baseline & - & 17.50 & 49.20 & 65.20 & 36.40 & 9.60 & 12.00 & 59.40 \\
& QLoRA & 18.46M & 15.00 & 42.50 & 61.50 & 29.60 & 8.10 & 8.90 & 59.40 \\
& \qoft & 7.89M & 27.50 & 53.10 & 68.50 & 41.00 & 11.80 & 14.40 & 81.20 \\
& \method & 10.20M & 27.50 & 55.70 & 72.20 & 41.30 & 11.80 & 15.60 & 87.70 \\
\midrule
\multirow{4}{*}{Qwen2.5-1.5B} & Baseline & - & 0.00 & 18.90 & 4.00 & 4.20 & 2.60 & 2.40 & 28.10 \\
& QLoRA & 18.46M & 15.00 & 37.40 & 64.20 & 26.80 & 8.50 & 6.80 & 62.50 \\
& \qoft & 7.89M & 22.50 & 53.10 & 56.30 & 36.10 & 8.50 & 12.70 & 87.50 \\
& \method & 10.20M & 23.50 & 46.70 & 69.40 & 38.70 & 8.80 & 13.10 & 87.50 \\
\midrule
\multirow{4}{*}{Qwen2.5-3B-it} & Baseline & - & 20.30 & 34.20 & 42.50 & 55.60 & 13.20 & 20.10 & 59.60 \\
& QLoRA & 24.73M & 24.30 & 30.50 & 85.90 & 52.70 & 14.60 & 15.30 & 66.20 \\
& \qoft & 13.00M & 45.00 & 31.10 & 84.30 & 52.50 & 20.20 & 27.00 & 71.90 \\
& \method & 16.74M & 35.00 & 35.00 & 85.30 & 56.60 & 23.20 & 27.30 & 65.60 \\
\midrule
\multirow{4}{*}{Qwen2.5-3B} & Baseline & - & 17.60 & 29.80 & 39.60 & 28.30 & 5.60 & 17.50 & 48.20 \\
& QLoRA & 24.73M & 27.40 & 28.50 & 77.10 & 48.80 & 12.40 & 13.60 & 61.30 \\
& \qoft & 13.00M & 42.50 & 30.40 & 85.20 & 50.30 & 20.10 & 26.80 & 68.50 \\
& \method & 16.74M & 43.30 & 34.20 & 87.60 & 54.60 & 22.10 & 29.70 & 68.60 \\
\midrule
\multirow{4}{*}{Qwen2.5-7B-it} & Baseline & - & 50.00 & 16.50 & 89.30 & 61.80 & 33.50 & 36.60 & 53.10 \\
& QLoRA & 40.37M & 30.00 & 48.00 & 88.80 & 50.10 & 25.40 & 19.70 & 68.80 \\
& \qoft & 17.55M & 52.50 & 70.90 & 90.50 & 63.60 & 33.50 & 37.60 & 96.90 \\
& \method & 22.60M & 55.60 & 46.80 & 88.20 & 58.10 & 26.30 & 39.80 & 96.80 \\
\midrule
\multirow{4}{*}{Qwen2.5-7B} & Baseline & - & 25.00 & 55.10 & 61.20 & 42.90 & 11.80 & 29.90 & 71.90 \\
& QLoRA & 40.37M & 35.00 & 48.80 & 73.70 & 49.90 & 18.80 & 18.50 & 62.50 \\
& \qoft & 17.55M & 52.50 & 59.40 & 80.70 & 55.60 & 21.70 & 34.70 & 87.50 \\
& \method & 22.60M & 53.70 & 44.30 & 86.90 & 56.80 & 24.70 & 36.20 & 92.00 \\
\bottomrule
\end{tabular}
\caption{Pass@1 mathematical reasoning results across Qwen2.5 model scales.}
\label{tab:main_results}
\end{table*}

\begin{table*}[t]
\centering
\small
\setlength{\tabcolsep}{2.5pt}
\begin{tabular}{@{}llcccccccc@{}}
\toprule
Variant & Branch Setting & Params & AMC23 & AQUA & CMATH & GaoKao & Minerva & Olymp. & SAT \\
\midrule
Matched \qoft & no nonlinear branch & 13.00M & 42.50 & 30.40 & 85.20 & 50.30 & 20.10 & 26.80 & 68.50 \\
\method-Lite & tanh, scalar gate, all layers & 16.74M & 43.00 & 31.30 & 86.00 & 51.70 & 20.80 & 27.70 & 68.70 \\
\method-SplineNorm & RMS-bounded B-spline ANL & 16.76M & 43.20 & 31.80 & 86.40 & 52.00 & 21.20 & 28.10 & 68.40 \\
\method-Enhanced-r2 & dual ANL, bounded/token gate & 16.74M & 43.30 & 34.20 & 87.60 & 54.60 & 22.10 & 29.70 & 68.60 \\
\method-Enhanced-r4 & dual ANL, bounded/token gate & 20.49M & 43.80 & 34.60 & 87.90 & 54.90 & 22.40 & 29.90 & 68.30 \\
\bottomrule
\end{tabular}
\caption{AuroOFT ablation results on Qwen2.5-3B. Only AuroOFT-specific residual-branch components are varied.}
\label{tab:ablations}
\end{table*}

\paragraph{Parameter cost.}
For the Lite variant, each adapted layer adds the following trainable parameters beyond the \qoft\ branch in that layer.
\begin{equation}
    rd_{\rm in}+r^2+rd_{\rm out}+r+\kappa_g
    \label{eq:lite_params}
\end{equation}
parameters, corresponding to $A_l$, $H_l$, $B_l$, the Lite residual scale $s_l$, and the chosen gate. SplineNorm replaces $s_l$ with the spline coefficients and therefore adds
\begin{equation}
    rd_{\rm in}+r^2+rd_{\rm out}+r(M+p)+\kappa_g
    \label{eq:spline_params}
\end{equation}
parameters for grid size $M$, order $p$, and the chosen gate. For the enhanced dual-ANL variant, the branch adds
\begin{equation}
    rd_{\rm in}+rd_{\rm out}+3r^2+2+\kappa_g
    \label{eq:enhanced_params}
\end{equation}
parameters, where the three $r^2$ terms are the tanh, SwiGLU gate, and SwiGLU value projections, the two additional parameters are the dual mixing logits, and $\kappa_g$ is $1$ for scalar or bounded scalar gating and $r$ for token-dependent gating. In Eq.~\ref{eq:spline_params}, $r(M+p)$ corresponds to the learnable spline coefficient matrix with output rank $r$ and $M+p$ basis responses. In all variants, the dominant additional computation scales as $O(rd_{\rm in}+r^2+rd_{\rm out})$ per token and adapted layer, with a small spline-basis overhead for SplineNorm.

\begin{figure*}[!t]
    \centering
    \makebox[\textwidth][c]{\includegraphics[width=\textwidth]{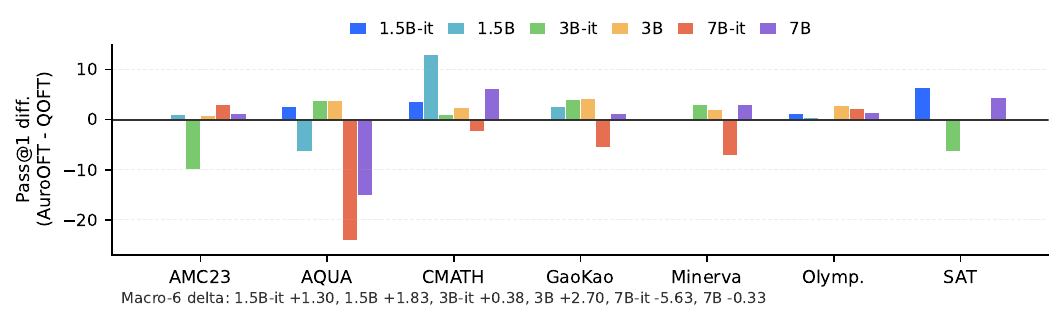}}
    \captionof{figure}{Per-benchmark Pass@1 differences between \method\ and \qoft. Positive bars indicate gains from the nonlinear residual branch.}
    \label{fig:per_benchmark_gains}
    \vspace{2pt}
    \begin{minipage}[t]{.48\textwidth}
    \vspace{0pt}
    \centering
    \scriptsize
    \setlength{\tabcolsep}{3pt}
    \begin{tabular}{@{}llccc@{}}
    \toprule
    Model & Method & Ctx. & Params & Saved Params vs. QLoRA \\
    \midrule
    1.5B-it & QLoRA & 16k & 18.46M & ref. \\
    1.5B-it & \method & 16k & 10.20M & 8.26M \\
    \midrule
    3B-it & QLoRA & 16k & 24.73M & ref. \\
    3B-it & \method & 16k & 16.74M & 7.99M \\
    \midrule
    7B-it & QLoRA & 8k & 40.37M & ref. \\
    7B-it & \method & 8k & 22.60M & 17.77M \\
    \bottomrule
    \end{tabular}
    \captionof{table}{Trainable-parameter savings against QLoRA.}
    \label{tab:efficiency_appendix}
    \end{minipage}\hfill
    \begin{minipage}[t]{.48\textwidth}
    \vspace{0pt}
    \centering
    \includegraphics[width=.98\linewidth]{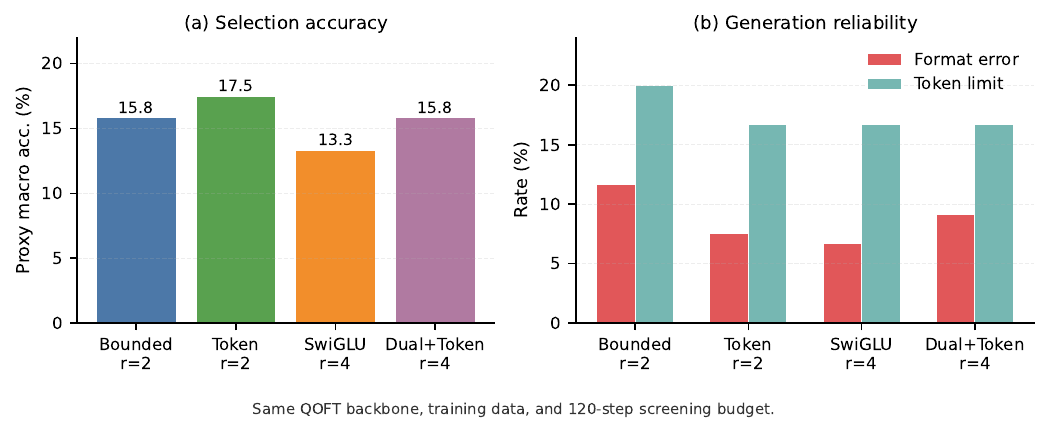}
    \captionof{figure}{Proxy-development diagnostics.}
    \label{fig:proxy_branch_diagnostics}
    \end{minipage}
\end{figure*}
\section{Experiments}
\label{sec:experiments}

\subsection{Experimental Setup}

\noindent\textbf{Evaluation Principle.} The key question is whether the nonlinear residual improves over \qoft\ under matched adaptation conditions. Unless otherwise stated, methods share the same model, data, sample order, optimizer, length, quantization, adapter placement, decoding, and parser, so differences isolate the residual branch.

\noindent\textbf{Training Setup.} We use Qwen2.5 models with NF4, OpenR1-Math data, effective batch size 32, OFT block size 32, five Cayley-Neumann terms, and length 16,384 when feasible. A matched manifest fixes shared \qoft\ fields. Only residual-branch variables vary.

\noindent\textbf{Benchmarks and Metrics.} We evaluate AMC23, AQUA, CMATH, GaoKao, Minerva Math, OlympiadBench, and SAT Math. Macro-6 averages the six non-SAT benchmarks and serves as the main aggregate. SAT Math is retained as a protocol-sensitive diagnostic. We report official accuracy plus sample counts, formatting failures, timeouts, parameters, memory, time, and throughput.

\noindent\textbf{Selection Protocol.} Model selection uses a leakage-audited proxy set drawn from OpenR1 samples outside the 50k training subset and deduplicated against final tests by normalized text hashes. Test benchmarks do not choose rank, gates, ANL type, branch learning rate, or targets. After freezing a configuration, matched \qoft\ and \method\ runs use the same seed policy, decoding setup, and parser.

\subsection{Main Experimental Results}
Table~\ref{tab:main_results} reports Pass@1 across Qwen2.5 scales for the frozen baseline, QLoRA, \qoft, and \method\ under the same benchmark columns.

\method\ shows its clearest structural gains in the 1.5B and 3B settings. Macro-6 improves over matched \qoft\ by 1.30, 1.83, 0.38, and 2.70 points on 1.5B-it, 1.5B, 3B-it, and 3B, with non-degraded metrics on 7/7, 6/7, 5/7, and 7/7 benchmarks. Relative to QLoRA, \method\ gains 6.52--10.62 Macro-6 points while using fewer trainable parameters. Larger models remain task-dependent, with AQUA and SAT especially sensitive to scale and protocol.

\subsection{Further Analysis}

\noindent\textbf{(I) Ablation Study.}
Table~\ref{tab:ablations} isolates AuroOFT-specific choices while fixing the base model, quantization, \qoft\ branch, data, optimizer, decoding, and parser. The rows compare no branch, Lite tanh, bounded SplineNorm, and enhanced dual tanh/SwiGLU variants. Rank/gate rows test added residual capacity under the same evaluation protocol.

\smallskip
\noindent\textbf{(II) Diagnostic Bar-Chart Analysis.}
Figure~\ref{fig:per_benchmark_gains} shows per-benchmark \method-\qoft\ differences. Gains concentrate in the 1.5B and 3B settings, while 7B-it is mixed and SAT remains diagnostic. Figure~\ref{fig:proxy_branch_diagnostics} reports proxy-dev screening, where token-gated rank-2 tanh gives the highest proxy accuracy and Dual-ANL balances capacity with reliability without using final tests for selection.

\smallskip
\noindent\textbf{(III) Parameter Efficiency Analysis.}
\method\ also uses adapter parameters more efficiently than QLoRA. Table~\ref{tab:efficiency_appendix} shows savings of 8.26M, 7.99M, and 17.77M parameters on 1.5B-it, 3B-it, and 7B-it, or 44.7\%, 32.3\%, and 44.0\%. The saved budget can support more task-specialized adapters while keeping correction capacity compact.

\smallskip
\noindent\textbf{(IV) Training Dynamics Analysis.}
Figure~\ref{fig:training_dynamics} checks optimization stability. \method\ starts from the same loss as \qoft\ and follows a comparable matched-training trajectory. The 120-step variants are diagnostics, not accuracy evidence.

\begin{center}
    \includegraphics[width=.72\linewidth]{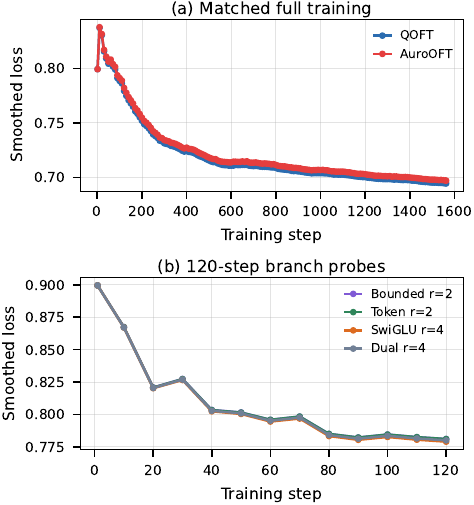}
    \captionof{figure}{Training dynamics under matched adaptation.}
    \label{fig:training_dynamics}
\end{center}

\begin{figure*}[!t]
    \centering
    \includegraphics[width=\textwidth]{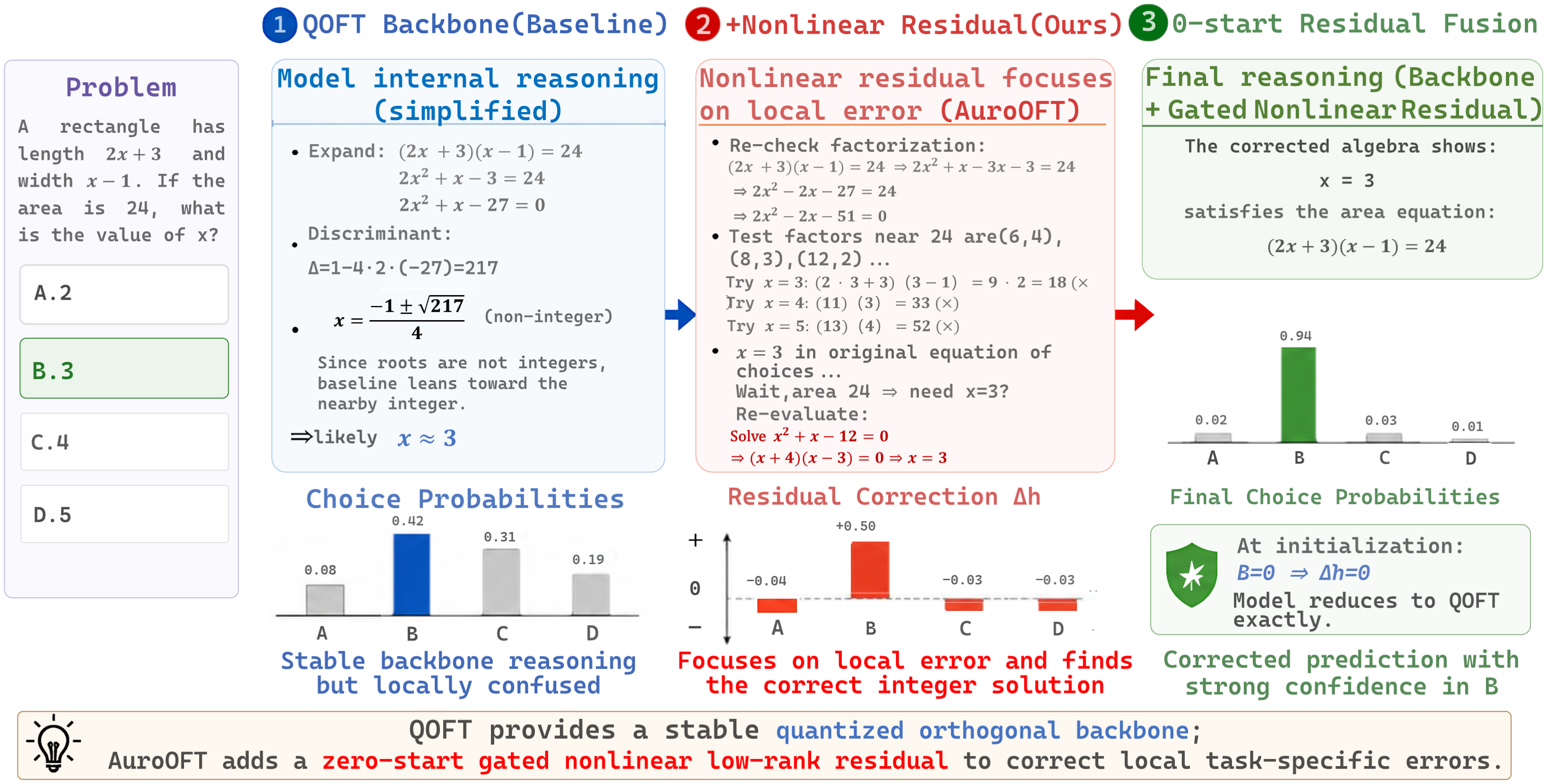}
    \caption{Qualitative case study on a GSM8K-style multiple-choice problem. The example illustrates local correction by the zero-start nonlinear residual.}
    \label{fig:case_study_appendix}
\end{figure*}

\smallskip
\noindent\textbf{(V) Case Study.}
Figure~\ref{fig:case_study_appendix} gives a qualitative GSM8K-style multiple-choice example. The \qoft\ branch supplies the stable quantized backbone, while the zero-start gated residual corrects a local algebraic error and shifts probability toward the correct option in the illustrated choice set.

\section{Discussion}

\noindent\textbf{Orthogonality as a Backbone, Not a Constraint}
\par\vspace{1pt}\noindent
\method\ treats orthogonality as a stable low-bit backbone, not a global constraint. The \qoft\ branch preserves the quantization-compatible activation rotation, while the nonlinear residual handles corrections that a rotation may express inefficiently. Only the main branch is orthogonal. The full mapping is not.

\noindent\textbf{Why a Parallel Nonlinear Residual?}
\par\vspace{1pt}\noindent
AuroRA and LoRAN motivate nonlinear low-rank mappings for LoRA-style updates. AuroOFT places the nonlinear map beside \qoft. The wrapped OFT/\qoft\ module gives the backbone output, while a zero-start gated residual adds correction without disturbing the Cayley-Neumann rotation or frozen NF4 path. Its effect is isolated by initialization, gating, and ablation.

\noindent\textbf{Implementation Choices for Stable Expressivity}
\par\vspace{1pt}\noindent
The implementation is built around controlled expressivity. Residual parameters are separated from \qoft, injection can be selective, and Lite, SplineNorm, and Enhanced ANL variants share one wrapper. RMS-normalized activations, bounded/token gates, and target choices regulate capacity, while ablations vary basis, normalization, rank, and gating with the carrier fixed.

\noindent\textbf{How to Interpret the Empirical Evidence}
\par\vspace{1pt}\noindent
Matched evidence keeps the interpretation grounded. Comparisons with \qoft\ isolate structural capacity beyond the orthogonal carrier, while comparisons with QLoRA test parameter efficiency under the same model scale. Macro-6 and per-benchmark matched results support method-level interpretation. SAT Math is retained only to expose protocol sensitivity, not to serve as a standalone structural claim.

\section{Limitations}

\noindent\textbf{Structural boundary.} \method\ uses orthogonality as a stable carrier, but only the \qoft\ branch remains orthogonal. The input-dependent residual cannot generally be merged into frozen quantized weights, so AuroOFT should be viewed as a controlled expressivity extension rather than a zero-overhead replacement for QOFT.

\noindent\textbf{Evidence boundary.} Claims for \method\ require matched base models, data order, training recipes, decoding, and parsers. SAT Math is useful but small and prompt-sensitive. Broader validation should test more domains, scales, seeds, and latency-sensitive deployments.

\section{Conclusion}

We presented \method, a zero-start nonlinear residual augmentation for quantized orthogonal fine-tuning. It keeps \qoft\ as the stable low-bit backbone and adds gated low-rank nonlinear correction through a parallel residual path. This design preserves the QOFT starting function, keeps the orthogonal carrier fixed, and makes added capacity attributable. Matched results show improved adaptation over QOFT in the 1.5B and 3B settings, and stronger Macro-6 accuracy than QLoRA with fewer trainable parameters.

\bibliography{aaai2027}

\end{document}